\documentclass{article}

\usepackage{natbib}
\usepackage{graphicx} 
\usepackage{tikz}
\usepackage{booktabs}
\usepackage[left=1.5cm, right=2cm]{geometry}
\newgeometry{left=3cm,bottom=4cm}
\usepackage{algorithm,algpseudocode}

\usepackage{subcaption}
\usepackage{tikz}
\usetikzlibrary{math}
\usetikzlibrary{3d}
\usetikzlibrary{backgrounds}
\usetikzlibrary{calc}
\usetikzlibrary{arrows}
\usetikzlibrary{shapes.geometric,arrows.meta, angles}
\usetikzlibrary{angles,quotes} 

\usepackage{tkz-euclide}
\usepackage{xcolor}
\usepackage{amsmath,amsthm, amssymb,amsfonts}
\usepackage{times}
\usepackage{latexsym}
\usepackage{verbatim}

\usepackage{pgfplots}
\pgfdeclarelayer{bg}
\pgfsetlayers{bg,main}

\newtheorem{theorem}{Theorem}

\newtheorem{assumption}[theorem]{Assumption}

\newtheorem{rem}{Remark}
\newtheorem{definition}[theorem]{Definition}

\newtheorem{lemma}[theorem]{Lemma}
\newtheorem{proposition}[theorem]{Proposition}

\def\bb0{{\mathbb{0}}}

\def\bb{{\mathbf{b}}}

\def\b0{{\mathbf{0}}}

\def\b1{{\mathbf{1}}}

\def\bbR{{\mathbb{R}}}

\def\cR{\mathcal{R}}

\def\cV{\mathcal{V}}

\def\cX{\mathcal{X}}

\def\sf0{{\mathsf{0}}}

\def\nn{\nonumber}
\def\dist{\text{dist}}
\def\PP{\mathcal{P}}
\def\GG{\mathcal{G}}
\def\Area{\textsf{Area}}
\def\Perim{\textsf{Perimeter}}
\def\Convexhull{\textsf{ConvexHull}}
\def\Tau{\mathcal{T}}

\tikzset{every node/.style={font=\footnotesize}}

\usetikzlibrary{intersections} 
\usetikzlibrary{calc,angles,quotes}

\def\thetaa{45}

\title{Breaking the $O(\sqrt{T})$ Cumulative Constraint Violation Barrier while Achieving $O(\sqrt{T})$ Static Regret in Constrained Online Convex Optimization}
\author{Haricharan Balasundaram, Karthick Krishna Mahendran, Rahul Vaze}
\date{}

\pgfplotsset{compat=1.18}

\begin{document}

\maketitle

\begin{abstract} 

The problem of constrained online convex optimization is considered, where at each round, once a learner commits to an action $x_t \in \cX \subset \bbR^d$, a convex loss function $f_t$ and a convex constraint function $g_t$ that drives the constraint $g_t(x)\le 0$ are revealed. The objective is to simultaneously minimize the static regret and cumulative constraint violation (CCV) compared to the benchmark that knows the loss functions and constraint functions $f_t$ and $g_t$ for all $t$ ahead of time, and chooses a static optimal action that is feasible with respect to all $g_t(x)\le 0$. In recent prior work \cite{Sinha2024}, algorithms with simultaneous regret of $O(\sqrt{T})$  and CCV of $O(\sqrt{T})$ or (CCV of $O(1)$ in specific cases \cite{vaze2025osqrttstaticregretinstance}, e.g. when $d=1$) have been proposed. It is widely believed that CCV is  $\Omega(\sqrt{T})$ for all algorithms that ensure that regret is $O(\sqrt{T})$ with the worst case input for any $d\ge 2$. In this paper, we refute this  and show that the algorithm of \cite{vaze2025osqrttstaticregretinstance} simultaneously achieves regret of $O(\sqrt{T})$ regret and CCV of $O(T^{1/3})$ when $d=2$. 

\end{abstract}

\section{Introduction}

In this paper, we consider the constrained version of the standard online convex optimization (OCO) framework, called constrained OCO or COCO. In COCO, on every round $t,$ the online algorithm first chooses an admissible action $x_t \in \mathcal{X} \subset \bbR^d$, and then the adversary chooses a convex loss/cost function $f_t: \mathcal{X} \to \mathbb{R}$ and a constraint function of the form $g_{t}(x) \leq 0,$ where $g_{t}: \mathcal{X} \to \mathbb{R}$ is a convex function. Let $\mathcal{X}^\star$ be the feasible set consisting of all admissible actions that satisfy all constraints $g_{t}(x) \leq 0, t\in [T]$. We work under the standard assumption that $\mathcal{X}^\star$ is not empty (called the {\it feasibility assumption}). 

Since $g_{t}$'s are revealed after the action $x_t$ is chosen, an online algorithm need not necessarily take feasible actions on each round, and in addition to the static regret 
\begin{equation} \label{eqn:intro-regret-def}
	\textrm{Regret}_{[1:T]} \equiv \sup_{\{f_t\}_{t=1}^T} \sup_{x^\star \in \mathcal{X^\star}} \textrm{Regret}_T(x^\star), ~\textrm{where~}\textrm{Regret}_T(x^\star) \equiv \sum_{t=1}^T f_t(x_t) - \sum_{t=1}^T f_t(x^\star),
\end{equation}
an additional metric of interest  is the total cumulative constraint violation (CCV) defined as 
  \begin{eqnarray} \label{eqn:intro-gen-oco-goal}
 	\textrm{CCV}_{[1:T]}  \equiv \sum_{t=1}^T \max(g_{t}(x_t),0) = \sum_{t=1}^T (g_{t}(x_t))^+. 
	\end{eqnarray}
The goal is to design an online algorithm to simultaneously achieve a small regret \eqref{eqn:intro-regret-def} with respect to any admissible benchmark $x^\star \in \mathcal{X}^\star$ and a small CCV \eqref{eqn:intro-gen-oco-goal}. 

With constraint sets ${\cal G}_t = \{x\in \cX : g_t(x)\le 0\}$ being convex for all $t$, and the assumption $\mathcal{X}^\star = \cap_t G_t  \neq \varnothing $  implies that sets $S_t = \cap_{\tau=1}^t {\cal G}_\tau$ are convex and are nested, i.e. $S_t\subseteq S_{t-1}$ and $\mathcal{X}^\star \in S_t$ for all $t$. Essentially, set $S_t$'s are sufficient to quantify the CCV.

\subsection{Prior Work}

\textbf{Constrained OCO (COCO): (A) Time-invariant constraints:} COCO with time-invariant constraints, \emph{i.e.,} $g_{t} = g, \forall \ t$ \citep{yuan2018online, jenatton2016adaptive, mahdavi2012trading, yi2021regret} has been considered extensively, where functions $g$ are assumed to be known to the algorithm \emph{a priori}. The algorithm is allowed to take actions that are infeasible at any time to avoid the costly projection step of the vanilla projected OGD algorithm and the main objective was to design an \emph{efficient} algorithm  with a small regret and CCV while avoiding  the explicit projection step. 

\textbf{(B) Time-varying constraints:} The more difficult question is solving COCO problem when the constraint functions, \emph{i.e.}, $g_{t}$'s, change arbitrarily with time $t$. In this setting, all prior work on COCO made the feasibility assumption. One popular algorithm for solving COCO considered a Lagrangian function optimization that is updated using the primal and dual variables \citep{yu2017online, pmlr-v70-sun17a, yi2023distributed}. Alternatively, \citet{neely2017online} and \cite{georgios-cautious} used the drift-plus-penalty (DPP) framework  \citet{neely2010stochastic} to solve the COCO, but which needed additional assumption, e.g. the Slater's condition in \citet{neely2017online} and with weaker form of the feasibility assumption \cite{neely2017online}'s. \citet{guo2022online} obtained the bounds similar to \citet{neely2017online} but without assuming Slater's condition. However, the algorithm \citet{guo2022online} was quite computationally intensive since it requires solving a convex optimization problem on each round. 

Finally, very recently, the state of the art guarantees on simultaneous bounds on regret $O (\sqrt{T})$ and CCV $O (\sqrt{T}\log T)$ for COCO were derived in \cite{Sinha2024} with a very simple algorithm that combines the loss function at time $t$ and the CCV accrued till time $t$ in a single loss function, and then executes the online gradient descent (OGD) algorithm on the single loss function with an adaptive step-size. Moreover, the result of \cite{Sinha2024} was shown to be tight in \cite[Lemma 6]{vaze2025osqrttstaticregretinstance} for an explicit input construction for which the algorithm of \cite{Sinha2024} has CCV of $\Omega(\sqrt{T}\log T)$ and that too for $d=1$. This was a consequence of the algorithm in \cite{Sinha2024} disregarding the geometry of the nested sets $S_t$'s and attempting to minimize both the regret and CCV for the worst case input.

A geometry-aware algorithm was proposed in \cite{vaze2025osqrttstaticregretinstance} that first takes an OGD step with respect to the most recently revealed loss function $f_{t-1}$ and then projects that on to the most recently revealed constraint set $S_{t-1}$. For this algorithm, an $O (\sqrt{T})$ regret bound and an instance specific CCV bound was established. In particular, the CCV was shown to be $O(1)$ when the sets are `nice' e.g., spheres or axis-aligned polygons, while in the general case, the CCV was shown to be $O(\cV)$, where $\cV$ is a parameter that depends on the distance between successive sets $S_t$'s and the shapes of sets $S_t$'s,  the dimension of the action space, and the diameter of the action space. Since no universal bound on $\cV$ was derived, CCV bound of $\min (\cV, O(\sqrt{T} \log T))$ was established by switching to the algorithm of  \cite{Sinha2024} in case $\cV$ exceeded $O(\sqrt{T})$. Thus, in the worst case, the bounds of \cite{Sinha2024} and \cite{vaze2025osqrttstaticregretinstance} are identical (regret of $O (\sqrt{T})$ and CCV of $O(\sqrt{T} \log T)$), however, for simple instances with $d=1$ for which the CCV bound of $O (\sqrt{T})$ \cite{Sinha2024} is tight, the CCV bound of \cite{vaze2025osqrttstaticregretinstance} is $O(1)$. Please refer to Table \ref{gen-oco-review-table} for a brief summary of the prior results. 

In a complementary direction, \cite{sinha2025tildeosqrtt} breached the CCV bound of $O(\sqrt{T})$ by trading it off with the regret. Specifically, an algorithm was proposed that achieves $\tilde O(\sqrt{dT} + T^\beta)$ regret and $\tilde O(d T^{1 - \beta})$ CCV, where $d$ is the dimension of the decision set and $\beta$ is a tunable parameter. This is achieved by a reduction to the constrained experts problem.

\begin{table}[t]
  \begin{tabular}{llllll}
    \toprule
    Reference  & Regret & CCV & Complexity per round\\
    \midrule
    \citet{neely2017online}  & $O(\sqrt{T})$ & $O(\sqrt{T})$ & Conv-OPT, Slater's condition \\
    \citet{guo2022online}  & $O(\sqrt{T})$ & $O(T^{\frac{3}{4}})$ & Conv-OPT \\
    \citet{yi2023distributed}  & $O(T^{\max(\beta, 1-\beta)})$ & $O(T^{1-\beta/2})$ & Conv-OPT  \\
    \citet{Sinha2024} & $O(\sqrt{T})$ & $O(\sqrt{T}\log T)$ & Projection \\  
    \cite{sinha2025tildeosqrtt} & $\tilde O(T^{\max \left(\beta, \tfrac 12 \right)})$ & $\tilde O(T^{1 - \beta})$ & $O(T^d)$ \\    
    \cite{vaze2025osqrttstaticregretinstance} &   $O(\sqrt{T})$ & $O(\min\{\cV,\sqrt{T}\log T\})$ & Projection \\
    \textbf{This Paper} (2-dimensions) & $O(\sqrt{T})$ & $O(T^{\tfrac 13})$ & Projection \\
    \bottomrule
  \end{tabular}
    
    \caption{Summary of the results on COCO for arbitrary time-varying convex constraints and convex cost functions. In the above table, $0\leq \beta \leq 1$ is an adjustable parameter. Conv-OPT refers to solving a constrained convex optimization problem on each round. Projection refers to the Euclidean projection operation on the convex set $\mathcal{X}$. The CCV bound for \cite{vaze2025osqrttstaticregretinstance} is stated in terms of $\cV$, which can be $O(1)$ or depend on the shape of convex sets $S_t$.}
    \label{gen-oco-review-table}
\end{table}

In comparison to the above discussed upper bounds, the best known simultaneous lower bound on regret and CCV \cite{Sinha2024}  for COCO is $\cR_{[1:T]} = \Omega(\sqrt{d})$ and $\text{CCV}_{[1:T]} = \Omega(\sqrt{d})$, where $d$ is the dimension of the action space $\cX$. Without constraints, $\cR_{[1:T]} = \Omega(\sqrt{T})$ for all online algorithms \cite[Theorem 3.2]{Hazanbook}, which trivially applies to COCO as well. Combining these two lower bounds by utilizing the lower bound from \cite[Theorem 3.2]{Hazanbook} in $\frac d2$ dimensions and the lower bound from \cite{Sinha2024} in the other $\frac d2$ dimensions yields a lower bound of $\cR_{[1:T]} = \Omega(\sqrt{T})$ and $\text{CCV}_{[1:T]} = \Omega(\sqrt{d})$ simultaneously.

\subsection{Main open question and Our Contribution}

The main open question in COCO is whether there exists an algorithm that can simultaneously achieve $\cR_{[1:T]} =O(\sqrt{T})$ and $\text{CCV}_{[1:T]} = o(\sqrt{T})$. Before this work, it was widely believed that this was not possible when $d\ge 2$. 

In this paper, we answer this question in the affirmative and show that \cite[Algorithm 2]{vaze2025osqrttstaticregretinstance} simultaneously achieves $\cR_{[1:T]} =O(\sqrt{T})$ and $\text{CCV}_{[1:T]} = O(T^{1/3})$ when $d = 2$. Even though our result holds only for $d=2$, it overcomes a fundamental bottleneck, and the analysis structurally improves upon the analysis of \cite{vaze2025osqrttstaticregretinstance} that also used similar geometric ideas. 

\cite[Algorithm 2]{vaze2025osqrttstaticregretinstance} is actually very simply: at time $t$, first take a OGD step with respect to most recently revealed function $f_{t-1}$, and then project that on to the most recent constraint set $S_t$. To derive our result, we exploit the fact that by taking projections from points in $S_{t-1}$ on to $S_t$, where $S_t$'s are nested, either the perimeter or the area of $S_t$ decreases {\it sufficiently} in each step compared to $S_{t-1}$ when $d=2$. Since both the area and diameter of the mother set $S_1$ is at most $D^2$ and $D$, respectively, we get our result. In contrast the analysis in \cite{vaze2025osqrttstaticregretinstance} bounded the decrease of {\it average width} \cite{eggleston1966convexity} going from $S_{t-1}$ and $S_t$ and derived an instance specific bound on the CCV that is valid for all $d$.

We conjecture is that in fact that the \cite[Algorithm 2]{vaze2025osqrttstaticregretinstance} has simultaneous $\cR_{[1:T]} =O(\sqrt{T})$ and $\text{CCV}_{[1:T]} = O(1)$ at least for $d=2$, and we need more fine grained analysis that will amortize CCV across time slots.
\section{COCO Problem}

We consider the COCO problem as defined in the Introduction, where the objective is to design online algorithms that simultaneously minimize static regret \eqref{eqn:intro-regret-def} and CCV \eqref{eqn:intro-gen-oco-goal}. We next state the standard assumptions made in the literature while studying the COCO problem \cite{guo2022online, yi2021regret, neely2017online, Sinha2024}.

\begin{assumption}[Convexity] \label{cvx}
$\mathcal{X} \subset \bbR^d$ is the admissible set that is closed, convex and has a finite Euclidean diameter $D$. The cost function $f_t: \mathcal{X} \mapsto \mathbb{R}$ and the constraint function $g_{t}: \mathcal{X} \mapsto \mathbb{R}$ are convex for all $t\geq 1$. 
\end{assumption}

\begin{assumption}[Lipschitzness] \label{bddness}
All cost functions $\{f_t\}_{t\geq 1}$ and the constraint functions $\{g_{t}\}_{ t\geq 1}$'s are $G$-Lipschitz, i.e., for any $x, y \in \mathcal{X},$ we have 
\begin{eqnarray*}
|f_t(x)-f_t(y)| \leq G||x-y||,~|g_{t}(x)-g_{t}(y)| \leq G||x-y||, ~\forall t\geq 1.
\end{eqnarray*}
\end{assumption}

\begin{assumption}[Feasibility] \label{feas-constr}
With ${\GG}_t = \{x\in \cX : g_t(x)\le 0\}$, we assume that $\mathcal{X}^\star = \cap_{t=1}^T \GG_t \neq \varnothing$. Any action $x^\star \in \cX^\star$ is defined to be feasible.
\end{assumption}
The feasibility assumption distinguishes the cost functions from the constraint functions and is common across all previous literature on COCO \cite{guo2022online, neely2017online, yu2016low,yuan2018online,yi2023distributed, georgios-cautious,Sinha2024}. 

Recall that ${\cal G}_t = \{x\in \cX : g_t(x)\le 0\}$ and sets $S_t = \cap_{\tau=1}^t {\cal G}_\tau$ are convex and are nested, i.e. $S_t\subseteq S_{t-1}$ and $\mathcal{X}^\star \in S_t$ for all $t$. Next, we define the projection of $x$ onto a set $\chi$ to be $\PP_\chi(x)$ and the projection distance to be $\dist(x,\chi)$.

\begin{definition} For a convex set $\chi$ and a point $x\notin \chi$, 
\begin{align}
 \PP_\chi(x) &= \arg\min_{y\in \chi} || x-y|| \\
 \dist(x,\chi) &= \min_{y\in \chi} || x-y|| = ||x - \PP_\chi(x)||
\end{align}
\end{definition}
Throughout, we take $||\cdot ||$ to be the $\ell_2$ norm. Since various $\ell_p$--norms have an at most $\sqrt{d} = O(1)$ looseness, this assumption does not violate the CCV bounds.

Since $g_t$ is Lipschitz, $[g_t(x)]^+ \le G \, \dist(x, \mathcal{G}_t) \le G \, \dist(x, S_t)$. This inequality is also tight in the sense that $g_t(x) = G \, \dist(x, S_t)$ itself is a valid constraint function (it is convex and $G$--Lipschitz) for a given set $S_t$. 
The total constraint violation upto to time $t$ for any algorithm is
\begin{align}\nn
\text{CCV}_{[1:t]} &= \sum_{\tau=1}^t [g_\tau(x_\tau)]^+ \\ \nn
&\le G\sum_{\tau=1}^t \text{dist}(x_{\tau}, S_{\tau}), \\ \nn
&= G \sum_{\tau=1}^t ||x_\tau - \PP_{S_t}(x_\tau)||, \\
&=G \sum_{\tau=1}^t ||p_{\tau}||,
\label{eqn:CCV_dist_relation}
\end{align}
where $p_\tau = \PP_{S_t}(x_\tau) - x_\tau \ \forall \ \tau \in [T]$.

We define $||p_\tau||$ to be the projection cost at time $\tau$ and controlling $\sum_{\tau=1}^T ||p_{\tau}||$ is thus sufficient to control $\text{CCV}_{[1:T]}$.

\section{Algorithm for solving COCO}

In this section, we recall the algorithm \cite[Algorithm 2]{vaze2025osqrttstaticregretinstance} for solving COCO.

\begin{algorithm}[H]
   \caption{Online Algorithm for COCO}
   \label{coco_alg_1}
\begin{algorithmic}[1]
   \State {\bfseries Input:} Sequence of convex cost functions $\{f_t\}_{t=1}^T$ and constraint functions $\{g_t\}_{t=1}^T,$ $G=$ a common Lipschitz constant,  $d$ dimension  of the admissible set $\mathcal{X},$ step size $\eta_t = \frac{D}{G \sqrt{t}}$. 
    $D=$ Euclidean diameter of the admissible set $\mathcal{X},$ $\mathcal{P}_\mathcal{X}(\cdot)=$ Euclidean projection operator on the set $\mathcal{X}$,      \State {\bfseries Initialization:} Set $ x_1 \in \mathcal{X}$ arbitrarily
   \State {\bf For} \ {$t=1:T$}
   \State \quad Play $x_t,$ observe $f_t, g_t,$ incur a cost of $f_t(x_t)$ and constraint violation of $(g_t(x_t))^+$
   \State \quad ${\cal G}_t = \{x\in \cX : g_t(x)\le 0\}$ and $S_t = \cap_{\tau=1}^t {\cal G}_\tau$ 
    \State \quad $y_{t + 1} =  \mathcal{P}_{S_{t-1}}\left(x_t - \eta_t \nabla f_t(x_t)\right)$
   \State \quad $x_{t+1} =  \mathcal{P}_{S_t}\left(y_{t+1}\right)$
   \State {\bf EndFor}
\end{algorithmic}
\end{algorithm}

Algorithm \ref{coco_alg_1} is essentially an online gradient descent (OGD algorithm) which first takes an OGD step from the previous action $x_{t}$ with respect to $f_{t}$ with appropriate step-size which is then projected onto $S_{t-1}$ to get $y_{t + 1}$, and then projects $y_{t + 1}$ onto the most recently revealed set $S_{t}$ to get $x_{t + 1}$, the new action to be played at time $t + 1$.

The following regret guarantee was derived in \cite{vaze2025osqrttstaticregretinstance} for Algorithm \ref{coco_alg_1}:
\begin{lemma}\cite[Lemma 7]{vaze2025osqrttstaticregretinstance} \label{lem:regretbound}
The $\textrm{Regret}_{[1:T]}$ for Algorithm \ref{coco_alg_1} is $O(\sqrt{T})$.
\end{lemma}

In the next section, we show that for $d=2$, $\textrm{CCV}_{[1:T]} = O(T^{1/3})$ for Algorithm \ref{coco_alg_1} fundamentally improving upon the worst case guarantee of $\textrm{CCV}_{[1:T]} = O(T^{1/2})$ derived in \cite{vaze2025osqrttstaticregretinstance}.
\color{black}

\subsection{Bounding $\text{CCV}_{[1:T]}$  \eqref{eqn:CCV_dist_relation} for $d=2$}
The main theorem of this paper is as follows.
\begin{theorem} \label{thm:theorem_CCV_bound} For $d = 2$, the CCV of Algorithm \ref{coco_alg_1} is upper bounded as
\begin{equation}
\mathrm{CCV}_{[1:T]} \le
\left(\frac 32 \sqrt{2} \pi \right)^{2/3} G T^{1/3} D \le 4 G T^{1/3} D.    
\end{equation}
\end{theorem}

The rest of the section is dedicated for proving Theorem \ref{thm:theorem_CCV_bound}.
To start, following \eqref{eqn:CCV_dist_relation}, without loss of generality, we restrict our attention to the indices $t$  for which $||p_t|| > 0$.
Next, with reference to Fig. \ref{fig:schematic_1}, we define some useful quantities.
\begin{enumerate}
    \item Recall that 
    \begin{equation}
        p_t = \PP_{S_t}(x_t) - x_t,        
    \end{equation}
    and that $\|p_t\|$ denotes the projection cost incurred at time $t$.

    \item Let $\ell_t$ be the line passing through $\PP_{S_t}(x_t)$ that is perpendicular to the vector $p_t$.

    \item Define $H_t$ as the closed half-space bounded by $\ell_t$ and not containing $x_t$. Formally,
    \begin{equation}
        H_t = \{z \in \mathbb{R}^2 : p_t \cdot z \ge p_t \cdot \left(\PP_{S_t}(x_t) \right) \},
    \end{equation}
    where $(\cdot)$ represents the dot product.

    \item Let $a_t$ and $b_t$ denote the two points at which the line $\ell_t$ intersects the body $S_{t-1}$ and let $w_t = \|a_t - b_t\|$.
\end{enumerate}

\begin{figure}[H]
\centering
\begin{tikzpicture}[scale=0.8,>=stealth]
  \clip (-6,-4) rectangle (6,6);

  \begin{scope}[rotate=\thetaa]

    \coordinate (origin) at (0,0);
    \coordinate (xt)     at (0,3.2);
    \coordinate (proj)   at (0,1.5);

    \def\RoutY{3.2}
    \def\RinY{1.5}
    \def\RoutLeftX{3.2}
    \def\RinLeftX{2.2}
    \def\RoutRightX{4.5}
    \def\RinRightX{3.5}

    \coordinate (rectSW) at (-\RoutLeftX-1,-10);
    \coordinate (rectNE) at (\RoutRightX+1,1.5); 

    \definecolor{lowerregion}{RGB}{231,76,60}
    \definecolor{outerhex}{RGB}{52,152,219}

    \coordinate (HtLeft)  at (-4.2, 1.5);
    \coordinate (HtRight) at ( 5.5, 1.5);

    \coordinate (Qtop)      at (1.0, 2.6);
    \coordinate (QrightUp)  at (\RoutRightX, 1.4);
    \coordinate (QrightLow) at (\RoutRightX,-1.4);
    \coordinate (Qbottom)   at (0,-2.4);
    \coordinate (QleftLow)  at (-\RoutLeftX,-1.4);
    \coordinate (QleftUp)   at (-\RoutLeftX,\RoutY);

    \path[fill=lowerregion!40, fill opacity=0.35]
      (rectSW) rectangle (rectNE);

    \node[lowerregion, left] at (HtRight) {$\ell_t$};

    \node[lowerregion] at (1,-5) {$H_t$};

    \path[fill=blue!20]
      (0,\RinY)
        arc[start angle=90, end angle=270,
            x radius=\RinLeftX, y radius=\RinY]
      --
      (0,-\RinY)
        arc[start angle=-90, end angle=90,
            x radius=\RinRightX, y radius=\RinY]
      -- cycle;

    \draw[name path=outerLeft, black, thick]
      (0,\RoutY) arc[start angle=90, end angle=270,
                     x radius=\RoutLeftX, y radius=\RoutY];

    \draw[name path=outerRight, black, thick]
      (0,-\RoutY) arc[start angle=-90, end angle=90,
                      x radius=\RoutRightX, y radius=\RoutY];

    \draw[black, thick]
      (0,\RinY) arc[start angle=90, end angle=270,
                    x radius=\RinLeftX, y radius=\RinY];

    \draw[black, thick]
      (0,-\RinY) arc[start angle=-90, end angle=90,
                     x radius=\RinRightX, y radius=\RinY];

    \draw[name path=htline, lowerregion, thick]
      (HtLeft) -- (HtRight);

    \path[name intersections={of=htline and outerLeft,  by=a_t}];
    \path[name intersections={of=htline and outerRight, by=b_t}];

    \node[circle,fill=lowerregion,inner sep=1pt, label=left:$a_t$]  at (a_t) {};
    \node[circle,fill=lowerregion,inner sep=1pt, label=above:$b_t$] at (b_t) {};

    \draw[black, thick,->] (xt) -- (proj) node[midway,above right] {$p_t$};

    \node[circle,fill,inner sep=1pt, label=left:$x_t$] at (xt) {};
    \node[circle,fill,inner sep=1pt, label=left:$\PP_{S_t}(x_t)$] at (proj) {};

    \node at (origin) {$S_t$};
    \node[] at ($(origin)+(-0.5,-2.4)$) {$S_{t-1}$};

    \coordinate (aw) at ([yshift=-10pt]a_t);
    \coordinate (bw) at ([yshift=-10pt]b_t);
    
    \draw[dashed, thin] (a_t) -- (aw);
    \draw[dashed, thin] (b_t) -- (bw);
    
    \draw[<->, thin] (aw) -- (bw) node[midway, below] {$w_t$};

  \end{scope}
\end{tikzpicture}
\caption{Regions $S_t$, $S_{t-1}$, half-plane $H_t$, the line $\ell_t$, and intersection points $a_t$, $b_t$.}
\label{fig:schematic_1}
\end{figure}
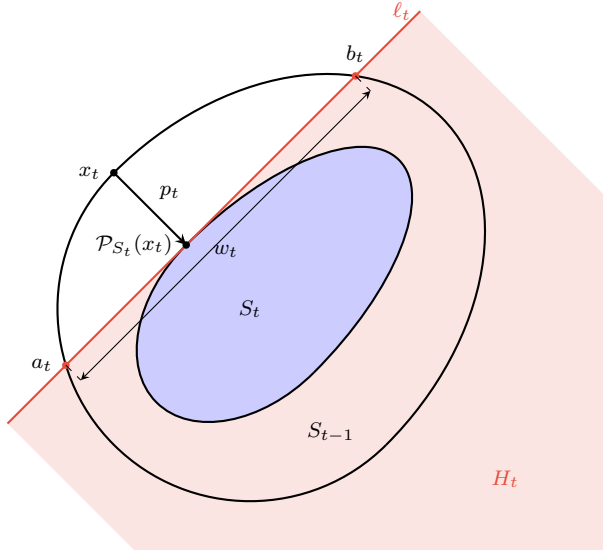

We note the following proposition which immediately follows since $\ell_t$ is perpendicular to $p_t$ and $\PP_{S_t}(x_t)$ lies on the boundary of $S_t$.

\begin{proposition}
$\ell_t$ is a supporting hyperplane (which in 2D is a line) for $S_t$ at $\PP_{S_t}(x_t)$.
\label{prop:supp_hyperplane}
\end{proposition}

We also note the below proposition which follows from the projection property (or equivalently, from Proposition~\ref{prop:supp_hyperplane}).
\begin{proposition}
    $S_t \subseteq H_t \ \forall \ t \in [T].$
    \label{prop:S_t_in_H_t}
\end{proposition}

\begin{rem} \label{remark:remark_on_area_perim}
Intuitively, since $S_t$'s are nested, with successive projections from sets $S_{t-1}$ to $S_t$ defining $||p_t||$, either the perimeter or the area of $S_t$ should decrease with respect to $S_{t-1}$ as a function of $||p_t||$ depending on the following two cases.
\begin{itemize}
    \item Case 1: $w_t$ in Fig.~\ref{fig:schematic_1} is small compared to $||p_t||$. In this case, the perimeter of $S_t$ decreases \textbf{sufficiently} compared to that of $S_{t-1}$. 
    \item  Case 2: $w_t$ in Fig.~\ref{fig:schematic_1} is large compared to $||p_t||$. In this case, the perimeter decrease is minimal, but complementarily, the area of $S_t$ decreases \textbf{sufficiently} compared to $S_{t-1}$. 
\end{itemize}
In the rest of the proof, we formalize this intuition by combining two complementary decreases appropriately. 
\end{rem}

\begin{rem}
A similar geometric idea was proposed in \cite{vaze2025osqrttstaticregretinstance} to bound the CCV of Algorithm \ref{coco_alg_1}, where instead of controlling the decrease in perimeter or area, the chosen metric was {\it average width} \cite{eggleston1966convexity}. By bounding the decrease in average width between $S_{t-1}$ and $S_t$ as a function of $||p_t||$,  an instance dependent bound on CCV was derived in \cite{vaze2025osqrttstaticregretinstance}. It is worth noting that when $d=2$, the average width $\displaystyle W(K) = \frac{\Perim(K)}{\pi}$ \cite{eggleston1966convexity}. The main novel idea that we bring out in this paper is that considering perimeter or average width alone is not sufficient to bound the CCV, and multiple complementary metrics are needed.
\end{rem}

\begin{definition}
\begin{equation}
\delta_t := \Perim(S_{t-1}) - \Perim(S_t)
\qquad\text{and}\qquad
\Delta_t := \Area(S_{t-1}) - \Area(S_t),    
\end{equation}
where $S_0 = \mathcal{X}.$
\end{definition}

We next note a basic result whose proof follows from Cauchy's surface area formula and is relegated to Appendix~\ref{appendix_proof_perim_subset}.

\begin{theorem}
Let $A,B \subset \mathbb{R}^2$ be bounded and closed convex sets. If $A \subseteq B$, then for $d = 2$ their perimeters satisfy
\begin{equation}
    \Perim(A) \leq \Perim(B).
\end{equation} \label{thm:perim_subset}
\end{theorem}
Since $S_t \subseteq S_{t - 1}$, it follows immediately that $\Area(S_t) \le \Area(S_{t - 1})$. From Theorem~\ref{thm:perim_subset} we have that $\Perim(S_t) \le \Perim(S_{t - 1})$. This yields the following proposition.
\begin{proposition} For all \(t \in [T]\), \(\delta_t \ge 0\) and \(\Delta_t \ge 0\); that is, both the perimeter and the area of \(S_t\) are non-increasing in \(t\). \label{prop:nonneg}
\end{proposition}

\begin{lemma}[Area decrease] For every $t \in [T]$, if the projection cost is $\|p_{t}\|$ and the width of $S_{t - 1}$ along the line onto which the action $x_t$ is projected onto is $w_{t}$, the area decrease $\Delta_t$ satisfies
\begin{equation}
    \Delta_t \ge \frac12 \| p_t \| w_t.
\end{equation}
\label{lem:area_decrease}
\end{lemma}
\begin{proof}
Recalling the definitions as illustrated in Fig. \ref{fig:schematic_1}, let $\Tau_t$ be the interior of the triangle with vertices $x_t$, $a_t$, and $b_t$, i.e., $\Tau_t$ is an open triangular region as depicted in Fig.~\ref{subfig:perim}.

$\Tau_t \subseteq S_{t - 1}$ follows since $S_{t - 1}$ is convex and $x_t, a_t, b_t, \in S_{t - 1}$. However, $\Tau_t \cap S_t = \emptyset$ since, by Proposition~\ref{prop:S_t_in_H_t}, any point in $S_t$ is also in $H_t$, while any point in $\Tau_t$ is in the open half-plane opposite to $H_t$. Since $\Tau_t, S_t \subseteq S_{t - 1}$ and $\Tau_t \cap S_t = \emptyset$, we get
\begin{equation}
    \Area(S_{t - 1}) \ge \Area(S_t) + \Area(\Tau_t)
\end{equation}
Using this,
\begin{equation}
\Delta_t = \Area(S_{t-1}) - \Area(S_t) \ge \Area(\Tau_t) = \frac12 \| p_t \| w_t.
\end{equation}
\end{proof}

\begin{lemma}[Perimeter Decrease] For every $t \in [T]$,
\begin{equation}
\delta_t \ge \sqrt{\| p_t \|^2 + \frac{w_t^2}{4}} - \frac{w_t}{2}.    
\end{equation}
\label{lem:perim_decrease}
\end{lemma}

\begin{figure}[!t]
\centering

\begin{subfigure}{0.48\textwidth}
\centering
\begin{tikzpicture}[scale=0.6]
  \clip (-6,-4) rectangle (6,6);

  \begin{scope}[rotate=\thetaa]
    \coordinate (origin) at (0,0);
    \coordinate (xt)     at (0,3.2);
    \coordinate (proj)   at (0,1.5);

    \def\RoutY{3.2}
    \def\RinY{1.5}
    \def\RoutLeftX{3.2}
    \def\RinLeftX{2.2}
    \def\RoutRightX{4.5}
    \def\RinRightX{3.5}

    \definecolor{lowerregion}{RGB}{231,76,60}
    \definecolor{outerhex}{RGB}{52,152,219}

    \coordinate (HtLeft)  at (-3.8, 1.5);
    \coordinate (HtRight) at (5.5, 1.5);

    \path[fill=blue!20]
      (0,\RinY) arc[start angle=90, end angle=270,
                    x radius=\RinLeftX, y radius=\RinY]
      --
      (0,-\RinY) arc[start angle=-90, end angle=90,
                     x radius=\RinRightX, y radius=\RinY]
      -- cycle;

    \draw[name path=outerLeft,  black, thick]
      (0,\RoutY) arc[start angle=90, end angle=270,
                     x radius=\RoutLeftX, y radius=\RoutY];
    \draw[name path=outerRight, black, thick]
      (0,-\RoutY) arc[start angle=-90, end angle=90,
                      x radius=\RoutRightX, y radius=\RoutY];

    \path[fill=blue!20]
      (0,\RinY) arc[start angle=90, end angle=270,
                    x radius=2.2, y radius=\RinY]
      --
      (0,-\RinY) arc[start angle=-90, end angle=90,
                    x radius=3.5, y radius=\RinY]
      -- cycle;
    \draw[black, thick] (0,\RinY) arc[start angle=90, end angle=270,
                                      x radius=2.2, y radius=\RinY];
    \draw[black, thick] (0,-\RinY) arc[start angle=-90, end angle=90,
                                      x radius=3.5, y radius=\RinY];

    \draw[name path=htline, lowerregion, thick] (HtLeft) -- (HtRight);
    \node[lowerregion, left] at (HtRight) {$\ell_t$};

    \path[name intersections={of=htline and outerLeft,  by=a_t}];
    \path[name intersections={of=htline and outerRight, by=b_t}];

    \path[dotted, fill=green!40, draw=black, thick]
      (xt) -- (a_t) -- (b_t) -- cycle;

    \node[circle,fill=lowerregion,inner sep=1pt, label=left:$a_t$]  at (a_t) {};
    \node[circle,fill=lowerregion,inner sep=1pt, label=above:$b_t$] at (b_t) {};

    \draw[black, thick,->] (xt) -- (proj) node[midway,above right] {$p_t$};
    \node[circle,fill,inner sep=1pt, label=left:$x_t$] at (xt) {};
    \node[circle,fill,inner sep=1pt, label=left:\scalebox{.7}{$\PP_{S_t}(x_t)$}] at (proj) {};

    \node at (origin) {$S_t$};
    \node at ($(origin)+(-0.5,-2.4)$) {$S_{t-1}$};
  \end{scope}
\end{tikzpicture}
\caption{For Lemma~\ref{lem:area_decrease}, note that $\Area(S_{t - 1}) - \Area(S_t) \ge \Area(\Tau_t)$, where $\Tau_t$ is the interior of the triangle with vertices $x_t$, $a_t$, and $b_t$.}
\label{subfig:perim}
\end{subfigure}
\hfill
\begin{subfigure}{0.48\textwidth}
\centering
\begin{tikzpicture}[scale=0.6]
  \clip (-6,-4) rectangle (6,6);

  \begin{scope}[rotate=\thetaa]

    \coordinate (origin) at (0,0);
    \coordinate (xt)     at (0,3.2);
    \coordinate (proj)   at (0,1.5);

    \def\RoutY{3.2}
    \def\RinY{1.5}
    \def\RoutLeftX{3.2}
    \def\RoutRightX{4.5}

    \definecolor{lowerregion}{RGB}{231,76,60}

    \coordinate (HtLeft)  at (-3.8, 1.5);
    \coordinate (HtRight) at ( 5.5, 1.5);

    \begin{scope}
      \clip
        (0,\RoutY) arc[start angle=90, end angle=270,
                       x radius=\RoutLeftX, y radius=\RoutY]
      --
        (0,-\RoutY) arc[start angle=-90, end angle=90,
                         x radius=\RoutRightX, y radius=\RoutY]
      -- cycle;

      \clip (-10,-10) rectangle (10,1.5);

      \path[fill=brown!30] (-10,-10) rectangle (10,1.5);
    \end{scope}

    \draw[name path=outerLeft,  black, thick]
      (0,\RoutY) arc[start angle=90, end angle=270,
                     x radius=\RoutLeftX, y radius=\RoutY];
    \draw[name path=outerRight, black, thick]
      (0,-\RoutY) arc[start angle=-90, end angle=90,
                      x radius=\RoutRightX, y radius=\RoutY];

    \path[fill=blue!20]
      (0,\RinY) arc[start angle=90, end angle=270,
                    x radius=2.2, y radius=\RinY]
      --
      (0,-\RinY) arc[start angle=-90, end angle=90,
                    x radius=3.5, y radius=\RinY]
      -- cycle;
    \draw[black, thick] (0,\RinY) arc[start angle=90, end angle=270,
                                      x radius=2.2, y radius=\RinY];
    \draw[black, thick] (0,-\RinY) arc[start angle=-90, end angle=90,
                                      x radius=3.5, y radius=\RinY];

    \draw[name path=htline, lowerregion, thick] (HtLeft) -- (HtRight);
    \node[lowerregion, left] at (HtRight) {$\ell_t$};

    \path[name intersections={of=htline and outerLeft,  by=a_t}];
    \path[name intersections={of=htline and outerRight, by=b_t}];

    \coordinate (a_offset) at ($(a_t) - (0,0.25)$);
    \coordinate (proj_offset_a) at ($(proj) - (0,0.25)$);

    \draw[<->,thick]
        (a_offset) -- 
        node[pos=0.8, below] {$w_t^{(a)}$} 
        (proj_offset_a);
        
    \coordinate (b_offset) at ($(b_t) - (0,0.25)$);
    \coordinate (proj_offset_b) at ($(proj) - (0,0.25)$);

    \draw[<->,thick]
        (proj_offset_b) -- node[below] {$w_t^{(b)}$} (b_offset);
        
    \path[draw=black, thick] (xt) -- (a_t) -- (b_t) -- cycle;

    \node at ($(origin)+(-1.5,-1.8)$) {$U_t = H_t \cap S_{t-1}$};

    \node[circle,fill=lowerregion,inner sep=1pt, label=left:$a_t$]  at (a_t) {};
    \node[circle,fill=lowerregion,inner sep=1pt, label=above:$b_t$] at (b_t) {};
    \draw[black, thick,->] (xt) -- (proj) node[midway,above right] {$p_t$};
    \node[circle,fill,inner sep=1pt, label=left:$x_t$] at (xt) {};
    \node[circle,fill,inner sep=1pt, label=left:\scalebox{.7}{$\PP_{S_t}(x_t)$}] at (proj) {};

    \node at (origin) {$S_t$};

  \end{scope}
\end{tikzpicture}
\caption{For Lemma~\ref{lem:perim_decrease}, region $U_t = H_t \cap S_{t-1}$. Note that $\delta_t \ge \Perim(\Convexhull (U_t \cup \{x_t \})) - \Perim(U_t)$ $=\|x_t - a_t\| + \|x_t - b_t\| - \|a_t - b_t\|.$}
\label{subfig:area}
\end{subfigure}

\caption{Schematics for the proofs of Lemma~\ref{lem:area_decrease} and~\ref{lem:perim_decrease}.}
\end{figure}

\begin{proof} Consider the set $U_t = S_{t - 1} \cap H_t$, which is convex. Since $S_t \subseteq S_{t - 1}$ and $S_t \subseteq H_t$ from Proposition~\ref{prop:S_t_in_H_t}, $S_t \subseteq U_t$ is true. From Theorem~\ref{thm:perim_subset}, we get 
\begin{equation}
    \Perim(S_t) \le \Perim(U_t).    
    \label{eqn:perim_less}
\end{equation}

Similarly, consider $\Convexhull(U_t \cup \{x_t \})$, which by definition is convex. Since $x_t \in S_{t - 1}$ and $U_t \subseteq S_{t - 1}$, $\Convexhull(U_t \cup \{x_t \}) \subseteq S_{t - 1}$. From Theorem~\ref{thm:perim_subset}, we get
\begin{equation}
    \Perim(\Convexhull(U_t \cup \{x_t \})) \le \Perim(S_{t - 1}).
    \label{eqn:perim_less_2}
\end{equation}

Using ~\eqref{eqn:perim_less} and~\eqref{eqn:perim_less_2}, we get:
\begin{equation}
    \delta_t = \Perim(S_{t - 1}) - \Perim(S_{t}) \ge \Perim(\Convexhull(U_t \cup \{x_t \})) - \Perim(U_t). \label{eqn:delta_t_defn}
\end{equation}

Consider the sets $U_t$ and $\Convexhull(U_t \cup \{x_t \})$ as shown in Fig.~\ref{subfig:area}. The boundaries of the two sets $U_t$ and $\Convexhull(U_t \cup \{x_t \})$ differ only in three segments. While $\Convexhull(U_t \cup \{x_t \})$ has the line segments $\overline{x_t a_t}$ and $\overline{x_t b_t}$ as part of the boundary, $U_t$ has $\overline{a_t b_t}$ as part of the boundary. The difference in their perimeters is thus exactly $\|x_t - a_t\| + \|x_t - b_t\| - \|a_t - b_t\|$.

Next, let $\|a_t - \PP_{S_t}(x_t)\| = w_t^{(a)}$ and 
$\|b_t - \PP_{S_t}(x_t)\| = w_t^{(b)}$, such that
$w_t^{(a)} + w_t^{(b)} = w_t$. Without loss of generality assume
$w_t^{(a)} \le w_t^{(b)}$. Note that
\begin{equation}
\|x_t - a_t\| = \sqrt{\| p_t \|^2 + \left(w_t^{(a)} \right)^2}
\qquad\text{and}\qquad
\|x_t - b_t\| = \sqrt{\| p_t \|^2 + \left(w_t^{(b)} \right)^2}.   
\label{eqn:pythagoras_defns}
\end{equation}

From \eqref{eqn:delta_t_defn}, we get:
\begin{align}\nn
    \delta_t &\ge \Perim(\Convexhull(U_t \cup \{x_t \})) - \Perim(U_t), \\ \nn
    &= \|x_t - a_t\| + \|x_t - b_t\| - \|a_t - b_t\| \\ \nn
    &\overset{(a)}{=} \sqrt{\| p_t \|^2 + \bigl(w_t^{(a)}\bigr)^2}
   + \sqrt{\| p_t \|^2 + \bigl(w_t^{(b)}\bigr)^2}
   - \bigl(w_t^{(a)} + w_t^{(b)}\bigr), \\ \nn
    &\overset{(b)}{\ge} \sqrt{\| p_t \|^2 + \bigl(w_t^{(a)}\bigr)^2} - w_t^{(a)},
\end{align}
where (a) follows from \eqref{eqn:pythagoras_defns} and (b) follows because $\sqrt{\| p_t \|^2 + (w_t^{(b)})^2} \ge w_t^{(b)}$ since $\|p_t\|^2 \ge 0$.

Now note that for any given $\| p_t \|$ and for any $u \ge 0$, the function $f(u) = \sqrt{\| p_t \|^2+u^2} - u$ decreases as $u$ increases. Since $w_t^{(a)} \le \frac{w_t}{2}$, the minimum value of $f(u)$ would happen at $u = \frac{w_t}{2}$ and we conclude that
\begin{equation}
\delta_t \ge \sqrt{\| p_t \|^2 + (w_t^{(a)})^2} - w_t^{(a)}
\ge
\sqrt{\| p_t \|^2 + \frac{w_t^2}{4}} - \frac{w_t}{2}.
\end{equation}
\end{proof}

We next lower bound the maximum of the area decrease and the perimeter decrease  by a function of $\|p_t\|$ for every value of $w_t$.

\begin{lemma}
$\forall t \in [T]$ and $\forall \alpha > 0$, 
\begin{equation}
    \max (\delta_t, \alpha \Delta_t) \ge \| p_t \|^\frac 32 \sqrt{\frac{\alpha}{D \alpha + 2}}.
\end{equation}
\label{lem:decrease_perim_area}
\end{lemma}

\begin{proof}
Combining the perimeter and area bounds from Lemma~\ref{lem:area_decrease} and~\ref{lem:perim_decrease}, we obtain
\begin{equation} \label{eqn:max_exp}
\max(\delta_t, \alpha \Delta_t) \ge \max\!\left(\sqrt{\| p_t \|^2 + \frac{w_t^2}{4}} - \frac{w_t}{2}, \frac{\alpha}{2} \| p_t \| w_t \right).
\end{equation}

In the above expression, for a fixed $\|p_t\|$, the first term in the maximum decreases as $w_t$ increases while the second term increases as $w_t$ increases. Thus, the minimum value of the maximum occurs precisely when $w_t$ is such that the two arguments are equal. This happens when
\begin{equation}
\sqrt{\| p_t \|^2 + \frac{w_t^2}{4}} - \frac{w_t}{2} = \frac{\alpha}{2} \| p_t \| w_t.    
\end{equation}
Solving this equation yields a unique solution of $w_t = 2\sqrt{\frac{\| p_t \|}{\alpha (2 + \alpha \| p_t \|)}}.$ Substituting this value back into either term of \eqref{eqn:max_exp} gives
\begin{equation}
\max(\delta_t, \alpha \Delta_t) \ge \| p_t \|^{3/2} \sqrt{\frac{\alpha}{2 + \alpha \| p_t \|}}.
\end{equation}
Finally, using $\| p_t \| \le D$, we obtain the desired result.
\end{proof}

\subsection{Completing the Proof of Theorem \ref{thm:theorem_CCV_bound}} 
\begin{proof} Recall that $S_0 = \mathcal{X}$ has a diameter of $D$. Since $S_t \subseteq S_0$ for all $t$, the total perimeter decrease satisfies
\begin{equation}
\sum_{t=1}^T \delta_t = \Perim(S_0) - \Perim(S_T) \le \Perim(S_0) \le \pi D,    \label{eqn:perim_bound}
\end{equation}
and the total area decrease satisfies
\begin{equation}
\sum_{t=1}^T \Delta_t = \Area(S_0) - \Area(S_T) \le \Area(S_0) \le \tfrac{\pi}{4} D^2. \label{eqn:area_bound}
\end{equation}
Using Proposition~\ref{prop:nonneg} and Equations \eqref{eqn:perim_bound} and \eqref{eqn:area_bound}, we get
\begin{equation}
\sum_{t=1}^T \max(\delta_t,\alpha \Delta_t) \le \left(\sum_{t=1}^T \delta_t \right) + \alpha \left(\sum_{t=1}^T \Delta_t \right) \le \pi D+ \frac{\alpha\pi}{4} D^2. \label{eqn:max_bound}
\end{equation}
We choose $\alpha = \tfrac{2}{D}$, which ensures that the upper bound on $\sum_{t=1}^T \max(\delta_t,\alpha \Delta_t)$ in \eqref{eqn:max_bound} is linear in $D$. Thus, we get
\begin{equation}
\sum_{t=1}^T \max(\delta_t,\alpha \Delta_t) \le \frac 32 \pi D.
\end{equation}
Lemma~\ref{lem:decrease_perim_area} gives
\begin{equation}
\max(\delta_t,\alpha \Delta_t) \ge \| p_t \|^{3/2} \sqrt{\frac{\alpha}{2 + D\alpha}} = \frac{1}{\sqrt{2D}} \| p_t \|^{3/2}. \label{eqn:max_lower_bound}
\end{equation}
Summing up \eqref{eqn:max_lower_bound} over $t \in [T]$ and using \eqref{eqn:max_bound} gives
\begin{equation}
\frac{1}{\sqrt{2D}} \sum_{t=1}^T \| p_t \|^{3/2}
\le \sum_{t=1}^T \max(\delta_t,\alpha\Delta_t) \le \frac 32 \pi D.
\end{equation}
Thus, we have $\sum_{t=1}^T \| p_t \|^{3/2} \le \frac 32 \sqrt{2} \pi  D^{3/2}$. Next, we use H\"{o}lder's inequality, which states that for $p,q \ge 1$ with $\tfrac{1}{p} + \tfrac{1}{q} = 1$, we have
$\sum_{t=1}^T |\alpha_t \beta_t| \le \|\alpha\|_p \|\beta\|_q$. Applying this with $\alpha_t = \| p_t \|$, $\beta_t = 1$, $p=\tfrac{3}{2}$, and $q=3$ gives
\begin{align}
\sum_{t=1}^T \| p_t \| &= \sum_{t=1}^T \| p_t \| \cdot 1
\le \left(\sum_{t=1}^T \| p_t \|^{3/2}\right)^{2/3} \left(\sum_{t=1}^T 1^3\right)^{1/3} = T^{1/3}\!\left(\sum_{t=1}^T \| p_t \|^{3/2}\right)^{2/3}.
\end{align}
Substituting the above expression into~\eqref{eqn:CCV_dist_relation} yields
\begin{equation}
\mathrm{CCV}_{[1:T]} \le
\left(\frac 32 \sqrt{2} \pi \right)^{2/3} G T^{1/3} D.    
\end{equation}
\end{proof}
\section{Conclusions}

In this work, we have showed for the first time that static regret of  $O(\sqrt{T})$ and CCV of $O(T^\frac 13)$ can be achieved simultaneously for COCO, even though our result is limited to when $d=2$. The best known result in prior work even for $d=2$ had static regret of  $O(\sqrt{T})$ and CCV of $O(\sqrt{T})$, with the prevailing wisdom that that is the best possible simultaneous guarantee for any $d\ge 2$. We achieved this fundamentally improved result by considering two `complementary' geometric metrics such that at least one of them decreases sufficiently when we take projections from arbitrary point on to nested convex sets of finite diameter. To the best of our knowledge this is a novel idea in the literature of COCO.

An immediate question is: can we do something similar for $d \ge 3$. Even though conceptually there are multiple complementary metrics that can be considered for $d\ge 3$, e.g. $d - 1$-dimensional surface area  $d$-dimensional volume,  average width etc., however, there seems to be no metric which maintains a linear relationship with the average width which was the case in $d = 2$.
Thus, the resulting bounds from these obvious metrics give CCV guarantee that is worse than $O(\sqrt{T})$. Thus, we need more sophisticated arguments for $d\ge 3$. 

For $d=2$, in fact we want to conjecture that the CCV is $O(1)$ for Algorithm \ref{coco_alg_1}. The derived result in this paper treats each time slot independently, and fails to exploit the sequential nature of the projections on to nested convex sets which seem to keep the CCV constant. To get a better bound on the CCV, amortizing the projection vectors ($p_t$'s) over time is likely necessary.

\bibliographystyle{abbrvnat}
\bibliography{main.bib} 

\begin{thebibliography}{18}
\providecommand{\natexlab}[1]{#1}
\providecommand{\url}[1]{\texttt{#1}}
\expandafter\ifx\csname urlstyle\endcsname\relax
  \providecommand{\doi}[1]{doi: #1}\else
  \providecommand{\doi}{doi: \begingroup \urlstyle{rm}\Url}\fi

\bibitem[Eggleston(1966)]{eggleston1966convexity}
H.~G. Eggleston.
\newblock Convexity, 1966.

\bibitem[Guo et~al.(2022)Guo, Liu, Wei, and Ying]{guo2022online}
H.~Guo, X.~Liu, H.~Wei, and L.~Ying.
\newblock Online convex optimization with hard constraints: Towards the best of two worlds and beyond.
\newblock \emph{Advances in Neural Information Processing Systems}, 35:\penalty0 36426--36439, 2022.

\bibitem[Hazan(2019)]{Hazanbook}
E.~Hazan.
\newblock Introduction to online convex optimization.
\newblock \emph{CoRR}, abs/1909.05207, 2019.
\newblock URL \url{http://arxiv.org/abs/1909.05207}.

\bibitem[Jenatton et~al.(2016)Jenatton, Huang, and Archambeau]{jenatton2016adaptive}
R.~Jenatton, J.~Huang, and C.~Archambeau.
\newblock Adaptive algorithms for online convex optimization with long-term constraints.
\newblock In \emph{International Conference on Machine Learning}, pages 402--411. PMLR, 2016.

\bibitem[Liakopoulos et~al.(2019)Liakopoulos, Destounis, Paschos, Spyropoulos, and Mertikopoulos]{georgios-cautious}
N.~Liakopoulos, A.~Destounis, G.~Paschos, T.~Spyropoulos, and P.~Mertikopoulos.
\newblock Cautious regret minimization: Online optimization with long-term budget constraints.
\newblock In \emph{International Conference on Machine Learning}, pages 3944--3952. PMLR, 2019.

\bibitem[Mahdavi et~al.(2012)Mahdavi, Jin, and Yang]{mahdavi2012trading}
M.~Mahdavi, R.~Jin, and T.~Yang.
\newblock Trading regret for efficiency: online convex optimization with long term constraints.
\newblock \emph{The Journal of Machine Learning Research}, 13\penalty0 (1):\penalty0 2503--2528, 2012.

\bibitem[Neely(2010)]{neely2010stochastic}
M.~J. Neely.
\newblock Stochastic network optimization with application to communication and queueing systems.
\newblock \emph{Synthesis Lectures on Communication Networks}, 3\penalty0 (1):\penalty0 1--211, 2010.

\bibitem[Neely and Yu(2017)]{neely2017online}
M.~J. Neely and H.~Yu.
\newblock Online convex optimization with time-varying constraints.
\newblock \emph{arXiv preprint arXiv:1702.04783}, 2017.

\bibitem[Sinha and Vaze(2024)]{Sinha2024}
A.~Sinha and R.~Vaze.
\newblock Optimal algorithms for online convex optimization with adversarial constraints.
\newblock In \emph{The Thirty-eighth Annual Conference on Neural Information Processing Systems}, 2024.
\newblock URL \url{https://openreview.net/forum?id=TxffvJMnBy}.

\bibitem[Sinha and Vaze(2025)]{sinha2025tildeosqrtt}
A.~Sinha and R.~Vaze.
\newblock Beyond $\tilde{O}(\sqrt{T})$ constraint violation for online convex optimization with adversarial constraints, 2025.

\bibitem[Sun et~al.(2017)Sun, Dey, and Kapoor]{pmlr-v70-sun17a}
W.~Sun, D.~Dey, and A.~Kapoor.
\newblock Safety-aware algorithms for adversarial contextual bandit.
\newblock In \emph{International Conference on Machine Learning}, pages 3280--3288. PMLR, 2017.

\bibitem[Tsukerman and Veomett(2016)]{cauchy_surface_area_theorem}
E.~Tsukerman and E.~Veomett.
\newblock A simple proof of cauchy's surface area formula, 2016.
\newblock URL \url{https://arxiv.org/abs/1604.05815}.

\bibitem[Vaze and Sinha(2025)]{vaze2025osqrttstaticregretinstance}
R.~Vaze and A.~Sinha.
\newblock $o(\sqrt{T})$ static regret and instance dependent constraint violation for constrained online convex optimization, 2025.
\newblock URL \url{https://arxiv.org/abs/2502.05019}.

\bibitem[Yi et~al.(2021)Yi, Li, Yang, Xie, Chai, and Johansson]{yi2021regret}
X.~Yi, X.~Li, T.~Yang, L.~Xie, T.~Chai, and K.~Johansson.
\newblock Regret and cumulative constraint violation analysis for online convex optimization with long term constraints.
\newblock In \emph{International Conference on Machine Learning}, pages 11998--12008. PMLR, 2021.

\bibitem[Yi et~al.(2023)Yi, Li, Yang, Xie, Hong, Chai, and Johansson]{yi2023distributed}
X.~Yi, X.~Li, T.~Yang, L.~Xie, Y.~Hong, T.~Chai, and K.~H. Johansson.
\newblock Distributed online convex optimization with adversarial constraints: Reduced cumulative constraint violation bounds under slater's condition.
\newblock \emph{arXiv preprint arXiv:2306.00149}, 2023.

\bibitem[Yu and Neely(2016)]{yu2016low}
H.~Yu and M.~J. Neely.
\newblock A low complexity algorithm with $ o (\sqrt{T}) $ regret and $ o (1) $ constraint violations for online convex optimization with long term constraints.
\newblock \emph{arXiv preprint arXiv:1604.02218}, 2016.

\bibitem[Yu et~al.(2017)Yu, Neely, and Wei]{yu2017online}
H.~Yu, M.~Neely, and X.~Wei.
\newblock Online convex optimization with stochastic constraints.
\newblock \emph{Advances in Neural Information Processing Systems}, 30, 2017.

\bibitem[Yuan and Lamperski(2018)]{yuan2018online}
J.~Yuan and A.~Lamperski.
\newblock Online convex optimization for cumulative constraints.
\newblock \emph{Advances in Neural Information Processing Systems}, 31, 2018.

\end{thebibliography}

\appendix
\section{Proof of Theorem~\ref{thm:perim_subset}}
\label{appendix_proof_perim_subset}

\begin{proof}
For a convex body $K \subset \mathbb{R}^2$, its support function is defined by
\begin{equation}
h_K(u) := \max_{x \in K} \langle x, u \rangle, \qquad u \in S^1,    
\end{equation}
where $S^1$ is the unit circle and $\langle \cdot,\cdot \rangle$ is the standard Euclidean inner product.

The width of $K$ in direction $u$ is
\begin{equation}
w_K(u) := h_K(u) + h_K(-u), \qquad u \in S^1.    
\end{equation}
Next, we invoke Cauchy's Surface Area formula  (see, for instance, \cite{cauchy_surface_area_theorem})
\begin{equation}
\Perim(K) = \frac{1}{2} \int_{S^1} w_K(u)\, d\theta(u),    
\end{equation}
where $d\theta$ denotes the standard Lebesgue measure on the unit circle.

Now, let $A \subseteq B$. Then for every $u \in S^1$,
\begin{equation}
h_A(u) = \max_{x \in A} \langle x,u\rangle \le \max_{x \in B} \langle x,u\rangle = h_B(u).    
\end{equation}
Similarly, for $-u$ we obtain $h_A(-u) \le h_B(-u)$, and hence
\begin{equation}
w_A(u) = h_A(u) + h_A(-u) \le h_B(u) + h_B(-u) = w_B(u) \qquad\text{for all } u \in S^1.    
\end{equation}

Using Cauchy's formula for both $A$ and $B$ and the pointwise inequality $w_A \le w_B$, we get
\begin{equation}
\Perim(A) = \frac{1}{2} \int_{S^1} w_A(u)\, d\theta(u) \le \frac{1}{2} \int_{S^1} w_B(u)\, d\theta(u) = \Perim(B).
\end{equation}
\end{proof}

\end{document}